\documentclass[11pt]{article}

\usepackage[preprint]{acl}

\usepackage{times}
\usepackage{latexsym}

\usepackage[T1]
{fontenc}
\usepackage{booktabs}

\usepackage[utf8]{inputenc}

\usepackage{microtype}

\usepackage{inconsolata}

\usepackage{graphicx}
\usepackage{amsmath}
\usepackage{hyperref}
\usepackage{url}
\usepackage{booktabs}
\usepackage[table]{xcolor}  
\usepackage{bbm}
\usepackage{array}  
\usepackage{float}
\usepackage{mathabx}
\usepackage{amsmath}
\usepackage{algorithm2e}
\usepackage{titletoc}
\usepackage[titletoc]{appendix}
\usepackage{graphicx}
\usepackage{multirow}
\usepackage{subcaption}
\newtheorem{lemma}{Lemma}
\newtheorem{proof}{Proof}
\usepackage{cleveref}

\title{\textsc{MorphoBench}: A Benchmark with Difficulty Adaptive to Model Reasoning}

\author{\textbf{Xukai Wang{$^{1,4}$}\thanks{Contributed equally.}, Xuanbo Liu{$^{1,3}$}\footnotemark[1], Mingrui Chen{$^{1,4}$}\footnotemark[1], Haitian Zhong{$^{1,4}$}\footnotemark[1], Xuanlin Yang{$^{1,2}$}\footnotemark[1], Bohan Zeng{$^{2}$}\footnotemark[1],}\\ \textbf{Jinbo Hu{$^{2}$}\footnotemark[1], Hao Liang{$^{1,2}$}, Junbo Niu{$^{2}$}, Xuchen Li{$^{1,4}$}, Ruitao Wu{$^{1,3}$}, Ruichuan An{$^{2}$}, Yang Shi{$^{2}$}, } \\ \textbf{Liu Liu{$^{3}$}, Xu-Yao Zhang{$^{4}$}, Qiang Liu{$^{4}$}, Zhouchen Lin{$^{2}$}, Wentao Zhang{$^{1,2}$}\thanks{Corresponding authors: wentao.zhang@pku.edu.cn, dongbin@math.pku.edu.cn}, Bin Dong{$^{1,2  \dag}$}} \\ 
    {$^1$}Zhongguancun Academy \quad {$^2$}Peking University \quad {$^3$} Beihang University \\ {$^4$}School of Artificial Intelligence, University of Chinese Academy of Sciences \\
\url{https://github.com/OpenDCAI/MorphoBench}
}

\begin{document}
\maketitle
\begin{abstract}
With the advancement of powerful large-scale reasoning models, effectively evaluating the reasoning capabilities of these models has become increasingly important. However, existing benchmarks designed to assess the reasoning abilities of large models tend to be limited in scope and lack the flexibility to adapt their difficulty according to the evolving reasoning capacities of the models. To address this, we propose \textsc{MorphoBench}, a benchmark that incorporates multidisciplinary questions to evaluate the reasoning capabilities of large models and can adjust and update question difficulty based on the reasoning abilities of advanced models. Specifically, we curate the benchmark by selecting and collecting complex reasoning questions from existing benchmarks and sources such as Olympiad-level competitions. Additionally, \textsc{MorphoBench} adaptively modifies the analytical challenge of questions by leveraging key statements generated during the model's reasoning process. Furthermore, it includes questions generated using simulation software, enabling dynamic adjustment of benchmark difficulty with minimal resource consumption. We have gathered over 1,300 test questions and iteratively adjusted the difficulty of \textsc{MorphoBench} based on the reasoning capabilities of models such as o3 and GPT-5. \textsc{MorphoBench} enhances the comprehensiveness and validity of model reasoning evaluation, providing reliable guidance for improving both the reasoning abilities and scientific robustness of large models.
\end{abstract}

\section{Introduction}

\begin{figure}[h]
    \centering    \includegraphics[width=0.98\linewidth]{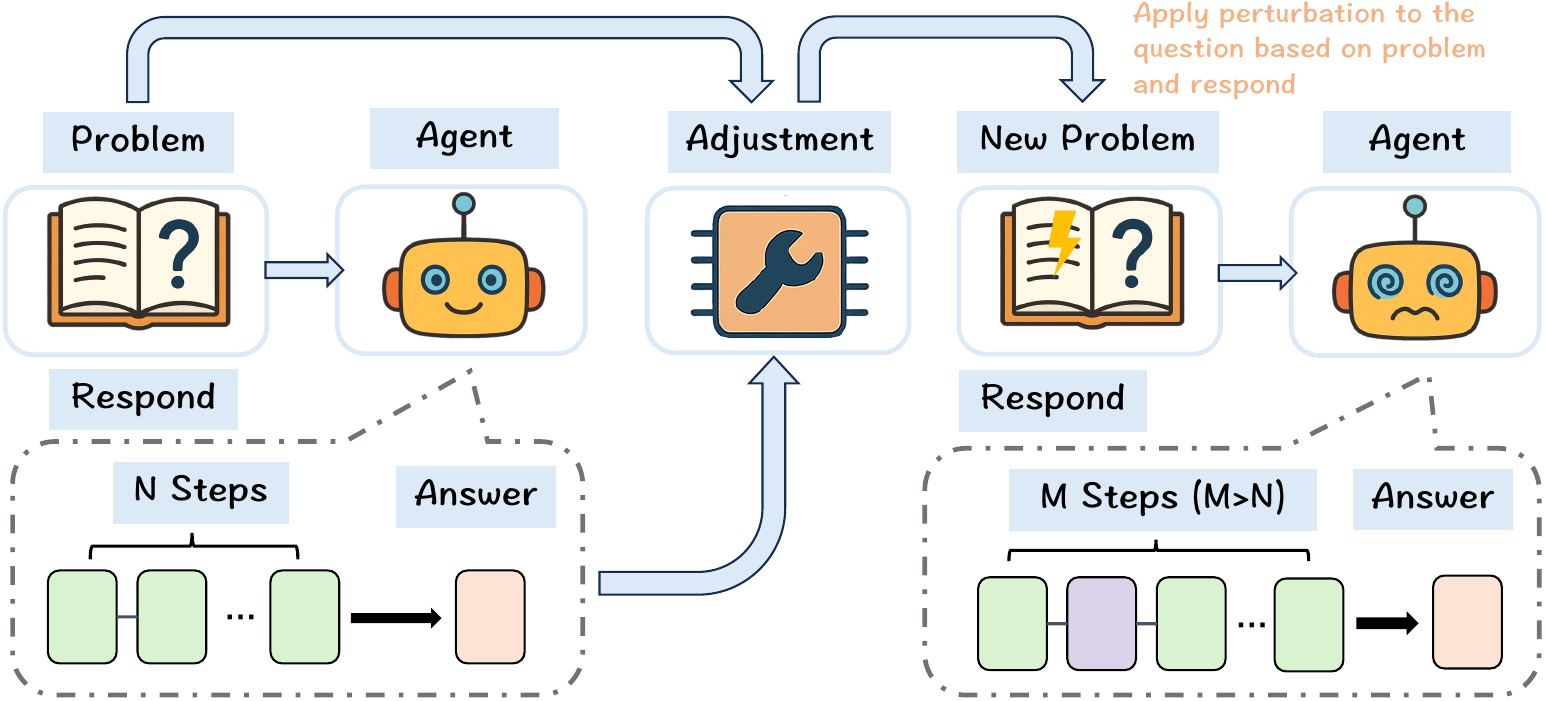}
    \caption{Overview of \textsc{MorphoBench}.}
    \label{fig:teaser}
    \vspace{-4mm}
\end{figure}

\begin{figure*}[h]
    \centering
    \includegraphics[width=0.89\linewidth]{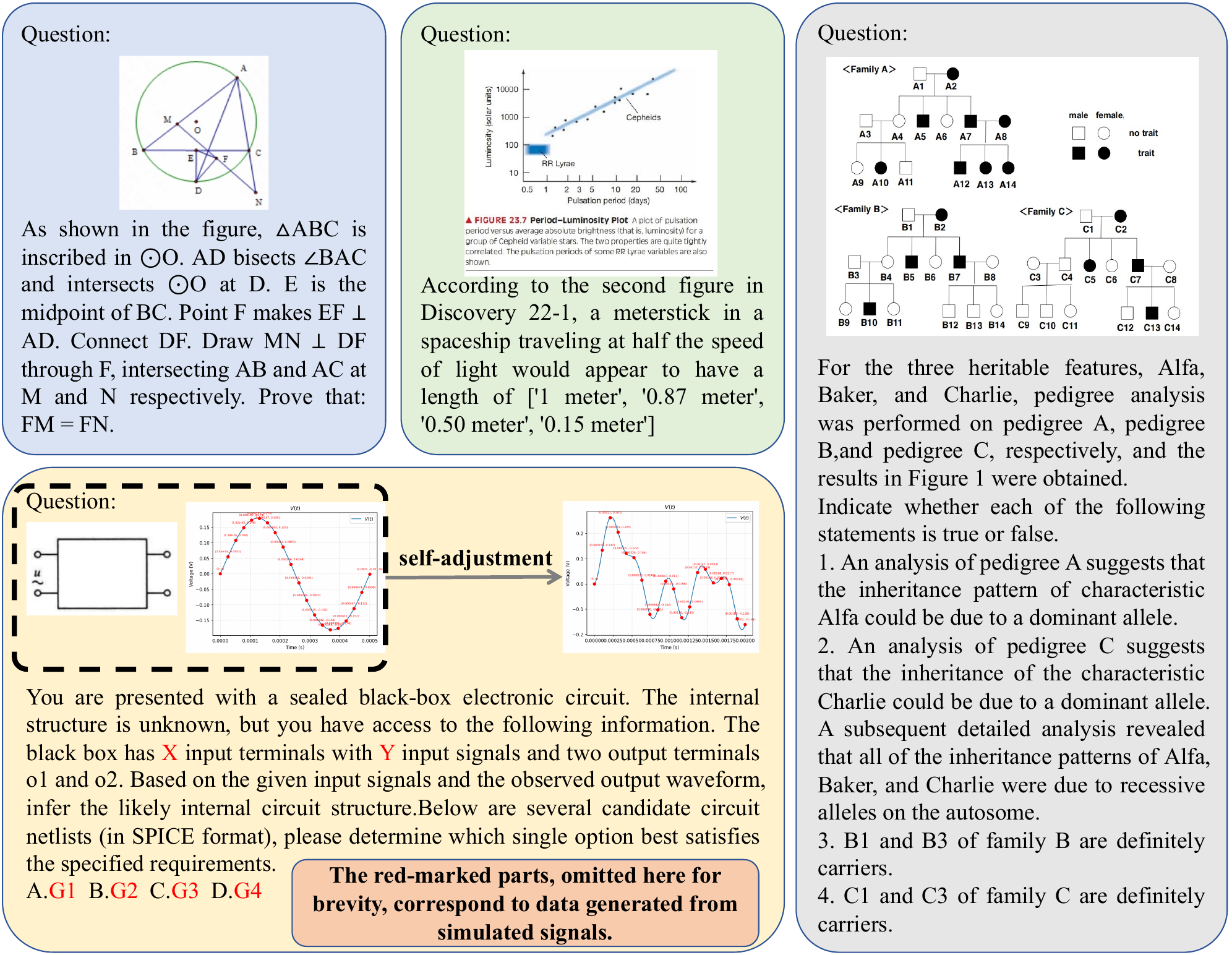}
    \caption{Testing examples from \textsc{MorphoBench}.}
    \label{fig:bench_examples}
    \vspace{-2mm}
\end{figure*}

In recent years, large-scale pre-trained models have achieved remarkable progress, demonstrating unprecedented capabilities across natural language processing, code generation, and multimodal understanding~\cite{devlin2019bert,achiam2023gpt,guo2024large,bai2023qwen,guo2025deepseek,chen2025sft}. Besides, there is a growing emphasis on strengthening their reasoning capabilities, especially in specialized academic domains such as mathematics, physics, logic, and related fields~\cite{zhou2024jiuzhang3,muennighoff2025s1,liu2023evaluating,xu2025medagentgym}. This shift reflects the broader ambition of artificial intelligence: to move from surface-level understanding to robust and generalizable reasoning.

To effectively evaluate large models, several benchmarks such as MME-Reasoning~\cite{yuan2025mme}, SeePhys~\cite{xiang2025seephys}, and HLE~\cite{phan2025humanity} have been proposed to measure reasoning abilities. Some models have even achieved gold-medal performance in competitions like the IMO~\cite{huang2025gemini25procapable} and IPHO~\cite{qiu2025physicssupernovaaiagent}. However, these benchmarks are static and cannot adapt to changes in a model’s reasoning proficiency. Moreover, although specialized agents may perform well in certain domains such as the IMO or IPHO, the coverage of current reasoning benchmarks remains narrow, as most focus on mathematics or physics problems. Even HLE~\cite{phan2025humanity}, while partially intended for reasoning assessment, still includes a large portion of simple or perception-based tasks rather than genuine multi-step reasoning.
Many existing benchmarks also rely on obscure or domain-specific knowledge, which tends to overestimate factual recall instead of true reasoning ability. Genuine reasoning should be evaluated through problems that involve complex logical inference based on simple or universally understood knowledge rather than the memorization of rare concepts. Therefore, a benchmark capable of dynamically adjusting difficulty according to a model’s reasoning ability, while covering multiple academic domains and emphasizing reasoning over knowledge rarity, is essential for accurate and stable evaluation.

To address these limitations, we propose \textsc{MorphoBench}, a multi-disciplinary reasoning benchmark with difficulty adaptive to model performance. Unlike existing benchmarks, \textsc{MorphoBench} dynamically adjusts question difficulty along two key dimensions: understanding conditions and constructing reasoning chains, enabling fair and comparable evaluation across models of different proficiency levels. It achieves this by modifying key statements within the model’s reasoning process, varying the clarity of problem conditions and introducing either guiding hints or distracting information to regulate reasoning complexity. We further conduct a fine-grained categorization and statistical analysis of problem types to support continuous updates and enhance multi-domain diversity. These designs position \textsc{MorphoBench} as a foundation for next-generation reasoning evaluation, fostering a transition from domain-specific competence to general reasoning toward AGI.

The main contributions of this paper can be summarized as follows:

\begin{itemize}
    \item We introduce \textsc{MorphoBench}, a novel benchmark that includes complex, reasoning-intensive problems in multiple disciplines. The benchmark supports adaptive difficulty calibration based on the model’s reasoning process, enabling fair and comparable evaluation across models with different levels of reasoning ability.

    \item \textsc{MorphoBench} changes the difficulty of evaluation questions along two dimensions: recognizing the given conditions and the reasoning process of the problem. \textsc{MorphoBench} identifies critical points in the model's problem-analysis process and adjusts questions accordingly. This approach leads to a more accurate and effective evaluation of reasoning capabilities.

    \item We also offer a more detailed breakdown of problem types and study how often each kind appears. This helps guide future updates to the benchmark. The design increases diversity and allows for a fuller assessment of model abilities, which moves us closer to AGI.
\end{itemize}

\begin{figure*}[t]
    \centering
    \includegraphics[width=0.92\linewidth]{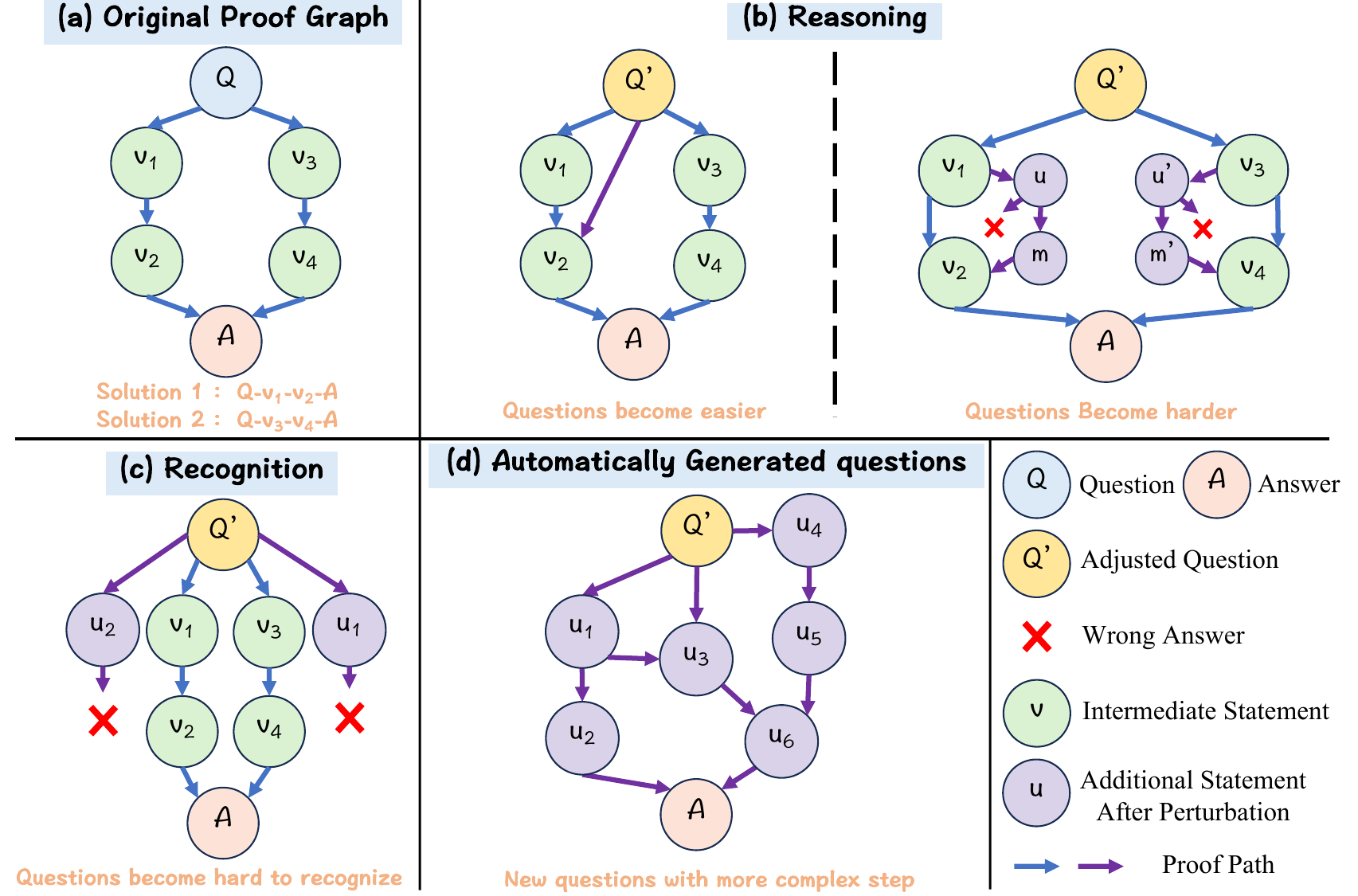}
    \caption{Demonstration of \textsc{MorphoBench}'s problem difficulty adjustment pipelines.}
    \label{fig:adjusted_pipes}
\end{figure*}

\section{Related Work}

\subsection{Large Models}

The Transformer architecture~\cite{vaswani2017attention} revolutionized AI by introducing self-attention, enabling efficient sequential processing and inspiring large-scale models. Subsequent works~\cite{radford2018improving, radford2019language, achiam2023gpt, bai2023qwen, touvron2023llama, brown2020language, nie2025large, zhu2025llada, liu2024deepseek, luo2024llm, shen2025let} expanded model scale to billions of parameters, achieving state-of-the-art NLP performance. Vision-Language Models~\cite{liu2023visual, wang2024qwen2, ye2024mplug, li2020unimo, li2022blip, li2023blip2, lin2023video, chen2024internvl, shi2025mavors, you2025llada, guo2025seed1, an2025unictokens, an2024mc, lin2025perceive} integrate vision and text, enabling multimodal understanding and generation. Recent efforts~\cite{openai2025o3, openai2025gpt5, comanici2025gemini, guo2025deepseek, guo2024deepseek, chen2025mint, su2025openthinkimg, bai2025multi, qiu2025physics, liang2025multimodal} enhance reasoning abilities, enabling logic, causality, and decision-making across complex tasks. These advances require sophisticated training for generalization. Thus, evaluating large-models capabilities remains crucial.

\subsection{Evaluation Benchmark for Large Models}

Evaluating large models requires robust benchmarks that truly reflect their capabilities~\cite{hendrycks2020measuring, wang2024mmlu}. As models evolve, specialized benchmarks for multimodal understanding and reasoning become increasingly necessary. The MME suite~\cite{lu2023mathvista, yu2023mm, yue2024mmmu, zhang2024vinoground, fu2025video, shi2025mme, hu2025video} addresses this by offering tasks that test integration and reasoning across visual and textual modalities. Further, reasoning-focused benchmarks~\cite{zheng2025livecodebench, yuan2025mme, guo2025r} evaluate complex reasoning tasks, while domain-specific ones~\cite{phan2025humanity, ruan2025mme, shen2025phyx, xiang2025seephys, li2024mmsci} assess specialized QA abilities. Yet, current benchmarks cannot adapt to models’ reasoning performance, making fair evaluation across varying capabilities a persistent challenge.

\section{\textsc{MorphoBench}}
\subsection{Data Collection}

To comprehensively evaluate the reasoning capabilities of large-scale models across disciplines, \textsc{MorphoBench} collects and standardizes questions requiring explicit reasoning from diverse academic sources, integrating questions from three sources to ensure coverage of diverse domains and reasoning styles, as shown in Fig.~\ref{fig:bench_examples}.

\textbf{(1) Open-source benchmarks.} Since several existing datasets already contain reasoning-oriented questions, we selectively incorporate such items 
 from \emph{Humanity’s Last Exam (HLE)} \cite{phan2025humanity} and \emph{MME-Reasoning} \cite{yuan2025mme}, which respectively provide 120 domain-spanning questions (physics, mathematics, computer science/AI, biology/medicine, and chemistry) and 100 questions targeting inductive, deductive, and abductive reasoning in multimodal settings. A subset of historical reasoning items is additionally drawn from \emph{HistBench}~\cite{qiu2025path}.

\textbf{(2) Olympiad-level competition problems.} To extend beyond existing benchmarks, which do not fully cover many challenging problems requiring complex reasoning, we collect high-difficulty questions across mathematics, physics, and chemistry. Specifically, mathematics items are drawn from competitions such as the Chinese Mathematical Olympiad (CMO), Putnam, IMO, and USAMO, while physics and chemistry problems are sourced from national Olympiads, including the Chinese Physics Olympiad (CPHO) and Chinese Chemistry Olympiad (CCO), as well as advanced high-school examinations.

\textbf{(3) Expert-designed complex reasoning scenarios.} We further construct new reasoning questions through automatic generation based on human-written templates, targeting tasks such as black-box circuit experiments or character recognition with distractors. The correct answers to these questions are determined by simulation software to ensure objectivity and reproducibility. The generation pipeline are described in Sec.~3.3.

All collected or generated questions from diverse disciplines were standardized following a unified style guide. Each question–answer pair underwent at least two rounds of expert review to verify accuracy, clarity, and metadata consistency. Ambiguous or low-quality items were removed after adjudication.


\subsection{Preliminary Analysis}

To illustrate how question difficulty can be systematically adjusted according to the reasoning capabilities of large-scale models, we first define the difficulty levels of questions in \textsc{MorphoBench}.
Recent LLMs increasingly demonstrate planning-like behaviors, 
outlining intermediate steps before producing the final solution.
\cite{gui2025hypertreeplanningenhancingllm,rawat2025preactmultistepplanningreasoning} Inspired by this observation, we formalize the solving process as a search problem on a \emph{directed proof graph}.\cite{wei2023chainofthoughtpromptingelicitsreasoning,yao2023treethoughtsdeliberateproblem} and analyze how the complexity of this graph, which reflects the model’s reasoning depth and branching structure, can be adjusted to control the difficulty of a question.

\subsubsection{Reasoning as Path Search in a Proof Graph}

For a reasoning question~$Q$, we construct a \textit{directed proof graph}
\begin{equation}
G_{Q} = (V,E,c).
\end{equation}

Each vertex $v \in V$ encodes an intermediate statement or subconclusion encountered during the reasoning process. Each directed edge $e = (v,v') \in E$ represents a single logically valid inference step. The edge weight $c(e)>0$ quantifies the expected computational cost, i.e., the difficulty for an LLM to move directly from state~$v$ to~$v'$ without additional intermediate reasoning statement.

The start vertex $s(Q)$ corresponds to the original problem statement, while the terminal vertex $t(Q)$ denotes the fully verified answer. For any path 
\begin{equation}
\pi = (v_{0},\dots,v_{k})
\quad\text{with}\quad
v_{0}=s(Q),\;v_{k}=t(Q),
\end{equation}
the accumulated cost is
\begin{equation}
\label{eq:cost}
\mathrm{Cost}(\pi) = \sum_{i=0}^{k-1} c\bigl(v_{i},v_{i+1}\bigr)
\end{equation}

The intrinsic difficulty of~$Q$ is defined as the expected cost of correctly deriving the answer over all valid reasoning paths from $s(Q)$ to $t(Q)$, weighted by their likelihoods under the model’s reasoning policy:
\begin{equation}
\label{eq:LQ}
\begin{aligned}
L(Q) 
&= \mathbbm{E}_{\pi \sim P(\pi \mid Q)}\bigl[\mathrm{Cost}(\pi)\bigr] \\
&= \sum_{\pi: s \to t} P(\pi \mid Q)\,\mathrm{Cost}(\pi)
\end{aligned}
\end{equation}

where $P(\pi \mid Q)$ denotes the model-assigned probability of following a valid reasoning path~$\pi$ given the question~$Q$.  
This expectation-based definition captures both the computational costs of individual inference steps and the diversity of plausible reasoning trajectories.

Intuitively, a direct jump from $v_{A}$ to $v_{B}$ may carry an extremely high cost, reflecting the model’s difficulty in performing a single, large inference leap. By contrast, a ``clever'' solution path, for example,
$v_{A}\to v_{1}\to v_{2}\to v_{B}$ may achieve a much lower total cost because it decomposes the reasoning into several simpler inference steps.

\begin{figure*}[t]
    \centering
    \includegraphics[width=0.92\linewidth]{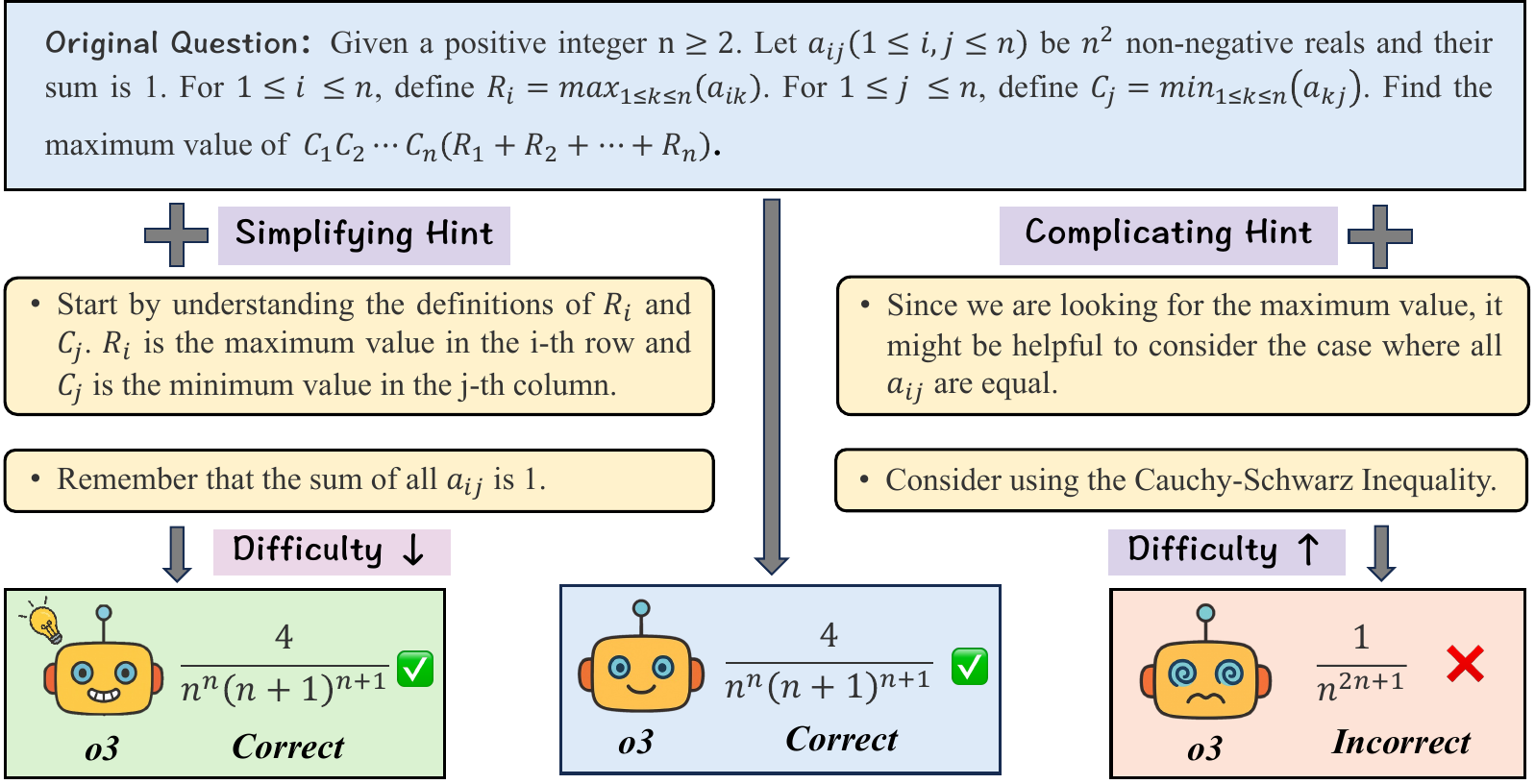}
    \caption{Different large models' reasoning results on \textsc{MorphoBench}.}
    \label{fig:exp_sample}
\end{figure*}

\subsubsection{Question Modification and Information Gap}
After defining the intrinsic difficulty of a question, We proceed to formalize the impact of question modification on reasoning difficulty.  

Let $\mathcal{R}$ be a modification algorithm that appends a hint $\tau$ to the original question, yielding $Q'=\mathcal{R}( Q,\tau)$. With respect to the target answer $A$, the \emph{information gap} of this modification is
\begin{equation}
\label{eq:deltaI}
\Delta I = K(A\mid Q')-K(A\mid Q),
\end{equation}
Here we use conditional Kolmogorov complexity $K(A\mid Q)$ to capture the effective complexity of producing the answer given the question, which directly corresponds to the model’s reasoning difficulty. Intuitively, the larger the information gap between $Q$ and $A$ , the more difficult the reasoning task becomes.  

A modification with $\Delta I \le 0$ is helpful or redundant, as it can reduce the search depth of the proof graph by providing intermediate constraints or decompositions, thereby lowering the path cost $L(Q)$; in contrast, $\Delta I > 0$ indicates a misleading or irrelevant adjustment.

Therefore, the following analysis primarily focuses on the effect of such adjustments with $\Delta I>0$ on the difficulty of solving the problem.

\subsubsection{Impact of Modifications on Reasoning Complexity}

To characterize how such misleading modifications increase reasoning difficulty, we define $Fail(Q,B)$ as the event that the agent exhausts budget $B$ before reaching $t(Q)$.

As detailed in Appendix~\ref{app:proofs}, the modification algorithm 
$\mathcal{R}$ can inject a large number of indistinguishable spurious outgoing edges into the proof graph, thereby inflating the cost of searching along the reasoning path.  
In this view, a positive information gap naturally corresponds to an expansion of the effective search space, since additional misleading edges increase the expected traversal cost along the optimal path.  
Building upon this abstraction, the misleading perturbations introduced by $\mathcal{R}$, together with the increased structural complexity of the graph,
imply that, for any fixed compute budget $B$, 
the failure probability of the perturbed problem is strictly larger than that of the original problem:
\begin{equation}
\Pr\bigl[Fail(Q',B)\bigr]-\Pr\bigl[Fail(Q,B)\bigr]\;>\;0.
\end{equation}
Based on the preliminary analysis of the difficulty adjustment, below we introduce the specific strategies we employ for this purpose.

\subsection{Difficulty Adaptation}
\label{sec:diff_adjust}

\paragraph{Adaptation based on agent reasoning.}
Shaping the agent reasoning process itself is a direct and effective way to control problem difficulty and widen the gap between question and answer. As shown in Fig.~\ref{fig:adjusted_pipes} (b), we adjust difficulty by introducing hints into key reasoning statements: simple hints lower difficulty, while complex hints raise it. To systematically manage this process, we construct the proof graph, where intermediate conclusions are modeled as lemmas. Lemma improvement operates in two ways: (1) adding or modifying hints at the lemma level, making certain reasoning steps either more explicit or more implicit; (2) structural operations, such as pruning lemmas to reduce exploration breadth or extending lemma chains to increase reasoning depth. The  algorithm not only enables dynamic control of problem complexity, but also makes lemma construction more interpretable and actionable, supporting finer-grained difficulty evolution.

\begin{table*}[htbp]
\centering
\begin{tabular}{l*{5}{c}}
\toprule
 & Mathematics & Engineering & Natural Sciences & Social Sciences & Other \\
\midrule
Total (share) & 552 (42.23\%) & 220 (16.83\%) & 250 (19.13\%) & 91 (6.96\%) & 194 (14.85\%) \\
Acc(\%) & 53.26 & 37.73 & 34.40 & 56.04 & 41.75 \\
\bottomrule
\end{tabular}
\caption{Subject-wise performance of \textit{o3} on \textsc{Morpho}‑v0. 
\textbf{Total (share)} indicates the number of questions and their proportion within the full dataset (N = 1307), 
while \textbf{Acc (\%)} reports the model’s accuracy for each subject category.}
\label{tab:subject_metrics_new}
\end{table*}

\paragraph{Adaptation based on agent recognition.}
\textsc{MorphoBench} increases the reasoning cost between questions and answers by perturbing the visual cues most critical to the model, making the model more prone to reasoning errors as illustrated in Fig.~\ref{fig:adjusted_pipes} (c). Instead of relying on predefined annotations, the model itself first indicates which elements it considers essential. These elements are then deliberately obfuscated at the text level, for example by introducing ambiguous wording or partially masking key terms, thereby hindering precise interpretation. Unlike random textual noise, such agent-driven perturbations directly target the linguistic features most relied upon, making them more challenging. If the model continues to answer correctly under these conditions, it demonstrates strong robustness and generalization; conversely, performance degradation reveals over-dependence on localized textual cues. This strategy thus provides a principled means of difficulty adjustment, testing whether the model remains effective when its key features are perturbed.

\begin{table*}[htbp]
  \centering
  \begin{tabular}{lccc|cc} 
    \toprule
    Model & \textsc{R}(Lite) & \textsc{Morpho}-v0 & \textsc{R}(Complex) & \textsc{Morpho}-v0$^*$ & \textsc{P}(Perturbed) \\
    \midrule
    claude4          & $33.55 \pm 1.66$ & $29.22 \pm 2.27$ & $20.88 \pm 1.43$  & $25.84 \pm 3.93$  & $22.90 \pm 3.77$  \\
    gemini-2.5-flash & $39.10 \pm 1.87$ & $35.65 \pm 2.60$ & $31.71 \pm 1.78$  & $38.24 \pm 4.37$  & $32.77 \pm 4.22$  \\
    gemini-2.5-pro   & $39.67 \pm 1.88$ & $34.66 \pm 2.58$ & $32.33 \pm 1.79$  & $36.76 \pm 4.33$  & $35.92 \pm 4.31$  \\
    gpt5             & $52.22 \pm 1.91$ & $45.33 \pm 2.70$ & $37.68 \pm 1.86$  & $48.95 \pm 4.49$  & $43.28 \pm 4.45$  \\
    grok4            & $29.70 \pm 1.61$ & $25.99 \pm 2.19$ & $23.79 \pm 1.50$  & $31.51 \pm 4.17$  & $28.57 \pm 4.06$  \\
    o3               & $48.24 \pm 1.92$ & $45.52 \pm 2.70$ & $35.85 \pm 1.84$  & $45.59 \pm 4.47$  & $40.55 \pm 4.41$  \\
    o4-mini          & $41.51 \pm 1.89$ & $37.72 \pm 2.63$ & $30.57 \pm 1.77$  & $46.22 \pm 4.48$  & $39.71 \pm 4.4$  \\
    \bottomrule
  \end{tabular}
  \caption{Model performance comparison across progressive versions of the \textsc{Morpho} benchmark.: 
  \textsc{Morpho-R}(Lite), \textsc{Morpho}-v0, \textsc{Morpho-R}(Complex), 
  \textsc{Morpho}-v0$^*$, and \textsc{Morpho-P}(Perturbed). 
  Here, \textsc{Morpho}-v0$^*$ refers to a subset containing only the 476 multimodal questions.}
  \label{tab:exp_results_diff}
\end{table*}

\paragraph{Adaptation for automatically generated questions.}
In \textsc{MorphoBench}, automatic question generation involves two central challenges: ensuring validity and regulating difficulty, as demonstrated in Fig.~\ref{fig:adjusted_pipes} (d). To guarantee validity, we incorporate external simulation software, such as circuit simulators, to systematically verify the correctness of generated outputs. To regulate difficulty, we adjust key generation parameters. Specifically, in circuit black-box tasks, difficulty is modulated by varying the number of exposed terminals, with a larger number increasing the complexity of inferring the internal structure. In “spot the different one” tasks, difficulty is controlled either by selecting character pairs with higher visual similarity or by expanding the grid size, thereby imposing greater demands on visual discrimination.
These mechanisms allow \textsc{MorphoBench} to evolve difficulty automatically: as terminal counts or grid complexity grow, the tasks become progressively harder. This enables continuous challenge for models and supports scalable evaluation of reasoning and multimodal understanding.

\subsection{Category Expansion}

To ensure broad coverage across disciplines, we assign structured attributes to problems and organize them into a three-level tree: task type (perception, retrieval, reasoning), knowledge dependence (closed, open, hybrid), and fine-grained skill categories (e.g., arithmetic, geometry, flow). This hierarchical design avoids over-concentration in a single dimension and makes the benchmark more representative. 

We iterate by setting per-leaf quotas and targeted collection for sparse leaves. This disciplined assignment and rebalance loop expands breadth while preserving difficulty structure, keeping benchmark diversity controllable over time.

\section{Experiment}

\subsection{Implementation Details}

To benchmark top-tier reasoning performance, we evaluate leading frontier models—Gemini-2.5-Flash, Gemini-2.5-Pro, GPT-5, Grok-4, Claude-4, and the OpenAI o-series (\textit{o3}, \textit{o4-mini})—which embody the current state of the art in complex reasoning and problem-solving.

We first benchmark all models on the original dataset \textsc{Morpho-v0} and perform discipline‑level analysis across mathematics, engineering, natural sciences, and social sciences. 
Then, three types of difficulty adaptation are applied on \textsc{Morpho-v0}:

\textbf{Agent-reasoning adaptation:}
We derive three variants from \textsc{Morpho-v0}: a simplified version \textsc{Morpho-R}(Lite) with lower reasoning complexity, and a challenging version, \textsc{Morpho-R}(Complex), where lemma hints are modified to control reasoning depth.

\textbf{Agent‑recognition adaptation:} from the original benchmark, we derived \textsc{Morpho‑P}(Perturbed) by perturbing critical textual and visual cues in 476 multimodal samples to assess model robustness under perception disturbance.  

\textbf{Automatic‑generation adaptation:} we further generated a series of graded circuit‑reasoning datasets, collectively denoted as \textsc{Morpho‑G}, by varying the number of terminals in black‑box circuit questions.  

\subsubsection{Evaluation Metric}

We assess model performance on \textsc{MorphoBench} and all its variants using \textbf{accuracy}.  
Accuracy measures the proportion of correctly answered questions and is defined as:
\begin{equation}
\label{eq:accuracy}
\text{Acc} \;=\; \frac{1}{N}\sum_{i=1}^{N} \mathbbm{1}\!\bigl[\hat{y}_{i}=y_{i}\bigr],
\end{equation}
where $N$ is the number of evaluated items, $y_{i}$ is the ground‑truth answer, and $\hat{y}_{i}$ is the model prediction.  
Answer correctness is automatically determined using the \textit{o3‑mini} model, ensuring a consistent and scalable evaluation across all difficulty variants.

\subsection{Main Comparison Results}

\begin{table*}[htbp]
  \centering
  \renewcommand{\arraystretch}{1.0}
  \setlength{\tabcolsep}{5pt}
  \resizebox{0.9\linewidth}{!}{%
    \begin{tabular}{ll*{10}{c}}
      \toprule
      & & \multicolumn{10}{c}{\textbf{Difficulty Level}} \\
      \cmidrule(lr){3-12}
      \textbf{Model} & \textbf{Metric} & 1 & 2 & 3 & 4 & 5 & 6 & 7 & 8 & 9 & 10 \\
      \midrule
      \textbf{o3} & Acc. (\%) & 48.3 & 30.0 & 48.0 & 23.1 & 40.7 & 39.3 & 54.2 & 57.7 & 44.0 & 34.8 \\
      \textbf{Gemini-2.5-Pro} & Acc. (\%) & 75.9 & 36.7 & 16.0 & 7.7 & 0.0 & 7.1 & 12.5 & 7.7 & 0.0 & 13.0 \\
      \bottomrule
    \end{tabular}
  }
  \caption{Model performance of o3 and Gemini-2.5 Pro on the \textsc{Morpho-G}. The circuit black-box problem is a single-choice question with six options in total.}
  \label{tab:different-level-eval-cir-horizontal}
  \vspace{-2mm}
\end{table*}

\paragraph{Cross-disciplinary reasoning performance.}
We selected o3, the best-performing model overall, as the representative for our cross-domain analysis. As shown in Table~\ref{tab:subject_metrics_new}, o3 attains the highest accuracy in social sciences (56.04\%), followed by mathematics (53.26\%) and other tasks (41.75\%), while its performance is lower in engineering (37.73\%) and natural sciences (34.40\%). This updated ordering highlights a more nuanced imbalance in cross-disciplinary reasoning: Frontier models exhibit strong robustness on tasks centered on textual representation and conceptual reasoning, while showing limitations when confronted with reasoning scenarios that demand symbolic derivation, precise quantitative manipulation, or domain-specific and expert-designed challenges. For a more detailed breakdown, see the Appendix~\ref{app:Cross-disciplinary}.

\paragraph{Influence of adjustment based on agent recognition and reasoning.}

To validate the effectiveness of our question modifications, we test existing state-of-the-art methods on the difficulty-adjusted data. As shown in Table~\ref{tab:exp_results_diff}, it is immediately evident that all models answer the questions more accurately when the questions become easier, while their performance deteriorates as the difficulty of the question increases.
Among them, o3 demonstrates the strongest performance, confirming its robust multimodal recognition and reasoning capabilities. However, although o3 outperforms GPT-5 on the original \textsc{MorphoBench} questions, GPT-5 exhibits a significantly smaller performance degradation when questions become more challenging, indicating that GPT-5 possesses more stable analytical abilities and knowledge reserves.

Additionally, results in Table~\ref{tab:exp_results_diff} clearly show that recognition-focused adjustments continue to affect model reasoning, though their impact remains smaller than that of adjustments targeting reasoning capacity. This suggests that in evaluations emphasizing strong reasoning skills, logical-level guidance exerts a greater influence on model thinking.

\paragraph{Influence of adjustment for automatically generated questions.}


For the circuit black-box tasks, we conducted evaluations on o3 and Gemini-2.5-Pro. Before testing, we systematically defined difficulty levels for black-box problems. Specifically, the difficulty was divided into ten levels based on the number of external terminals. Each level corresponds to the number of input terminals on the black box, which in turn specifies the number of alternating current (AC) voltages simultaneously applied to these terminals. As the number of terminals increases, the reasoning process becomes inherently more complex, resulting in progressively more challenging tasks. The experimental results are summarized in Table \ref{tab:different-level-eval-cir-horizontal}.

\begin{figure}[t]
    \centering
    \includegraphics[width=1\linewidth]{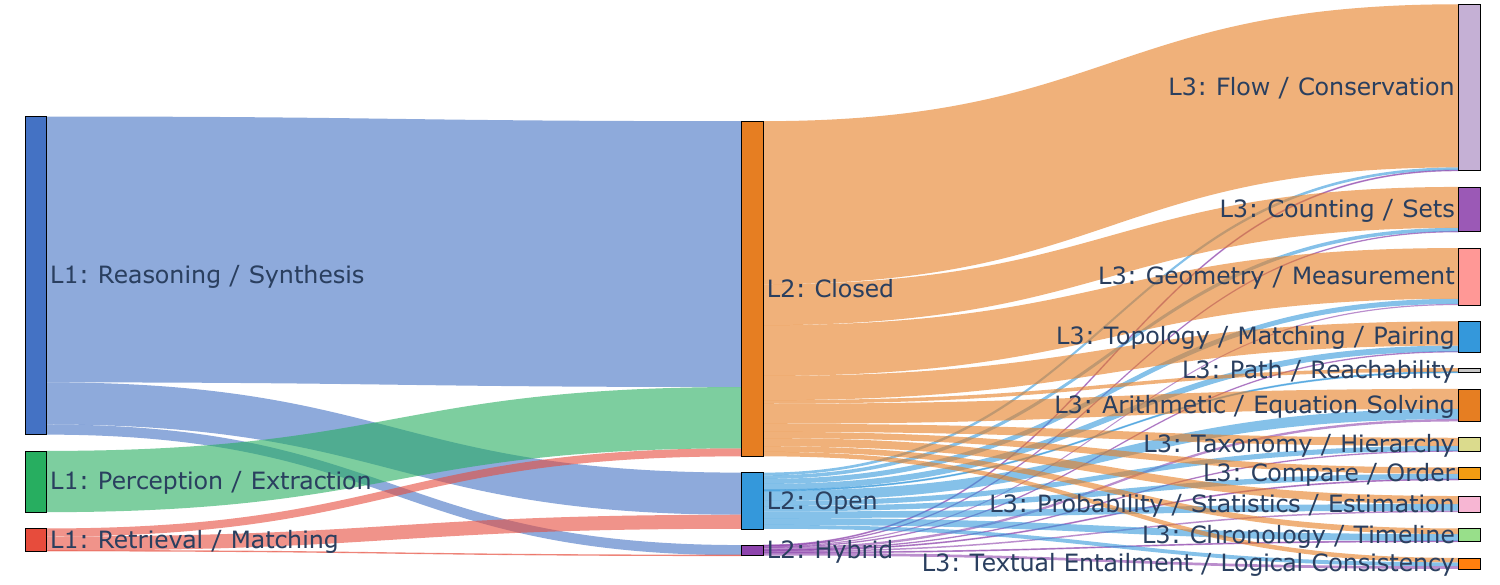}
    \caption{Diversity analysis of \textsc{MorphoBench}.}
    \label{fig:bench_diversity}
    \vspace{-4mm}
\end{figure}

As shown in the results, difficulty stratification strongly affects Gemini-2.5-Pro: as difficulty increases from level 1 to 10, its accuracy drops sharply from 75.9\% to 0–13\%, remaining low at higher levels. In contrast, o3’s accuracy fluctuates between 30\% and 58\% without a clear downward trend. This shows that the designed difficulty partition effectively suppresses Gemini-2.5-Pro’s performance, confirming the sensitivity of the difficulty design, while o3 exhibits weaker sensitivity. The difference likely results from distinct training distributions and inference strategies, as o3 can utilize external tools for analysis and problem solving, whereas Gemini-2.5-Pro aligns more closely with the intended progressive difficulty response.

\paragraph{Diversity analysis.}

The classification results in Fig.\ref{fig:bench_diversity} show that reasoning tasks are predominant, while all three top-level categories remain well represented. This ensures the benchmark includes both problems solvable through prompt-only evidence and those requiring external knowledge. At the leaf level, the dataset spans a diverse spectrum—from combinatorics and geometry to timeline reasoning and logical entailment.

Following our expansion and rebalancing operations, both hierarchical evenness and entropy show notable improvement, with leaf coverage reaching approximately 60\% of possible taxonomy paths. This validates both the taxonomy's expressiveness and the effectiveness of our balancing policy. For future iterations, we will prioritize problems with Open/Hybrid knowledge closure, retrieval-anchored items, and perception tasks requiring open knowledge. This strategy will help smooth the long-tail distribution while maintaining strong reasoning requirements.

\section{Conclusion}

In this paper, we introduce a new benchmark \textsc{MorphoBench}, which contains a wide variety of questions from multiple disciplines that demand strong reasoning capability. The difficulty of the questions can be adjusted according to the model's level of reasoning ability. Specifically, \textsc{MorphoBench} adjusts question difficulty by adding either positive or negative guidance at key stages of the analytical process, or by modifying the quality of critical information that the model needs to recognize. These adjustments are based on the model's performance during analysis. Additionally, we classify the attributes of the questions in \textsc{MorphoBench} in a more detailed manner, and improve the diversity and comprehensiveness of the benchmark by balancing these attributes across the dataset. We finally carry out rigorous experiments to validate the design and utility of \textsc{MorphoBench}.

\bibliography{custom}

\newpage

\appendix

\appendix
\appendixpage
\addappheadtotoc
\startcontents[sections]
\printcontents[sections]{l}{1}{\setcounter{tocdepth}{2}}
\newpage


\section{More Details about \textsc{MorphoBench}}

\subsection{Details of Taxonomy}


We organize each sample into a three-level taxonomy. For first mentions, we spell out the full name followed by its abbreviation in parentheses. The leaf category of any sample is given by the tuple $\langle \text{L1}, \text{L2}, \text{L3}\rangle$.

\paragraph{Level 1 (L1): Task Nature}
\begin{itemize}
  \item \textbf{Perception / Extraction (PERC).} Low-level understanding and signal extraction from inputs, including recognition, reading diagrams/OCR, locating entities, and basic counting.
  \item \textbf{Retrieval / Matching (RETR).} Locating or aligning information either provided in the prompt/evidence or drawn from external resources/commonsense; emphasis on correspondence and lookup.
  \item \textbf{Reasoning / Synthesis (RSYN).} Multi-step deduction or constraint satisfaction that integrates pieces of evidence (e.g., flow/conservation rules, multi-hop logic chains) to reach a conclusion.
\end{itemize}

\paragraph{Level 2 (L2): Knowledge Closure}
\begin{itemize}
  \item \textbf{Closed (CLO).} The answer is fully determined by the prompt and provided evidence; no outside knowledge is required.
  \item \textbf{Open (OPE).} Solving requires external knowledge beyond what is given (e.g., background facts, domain conventions).
  \item \textbf{Hybrid (HYB).} Primarily evidence-driven but benefits from a small amount of common or world knowledge (e.g., everyday conventions) to bridge gaps.
\end{itemize}

\paragraph{Level 3 (L3): Reasoning Primitive}
\begin{itemize}
  \item \textbf{Flow / Conservation (FLOW).} Applying conservation or balance principles (e.g., circuit KCL/KVL, mass/energy balance, network flow).
  \item \textbf{Path / Reachability (PATH).} Determining connectivity, routes, or shortest hops in graphs, mazes, or grids.
  \item \textbf{Chronology / Timeline (TIME).} Ordering events, aligning dates/eras/dynasties, or constructing consistent timelines.
  \item \textbf{Taxonomy / Hierarchy (TAXO).} Working with classification trees, phylogeny, or family hierarchies to place or infer relations.
  \item \textbf{Probability / Statistics / Estimation (PROB).} Handling uncertainty, intervals, likelihoods, sampling, or simple statistical summaries.
  \item \textbf{Arithmetic / Equation Solving (ARITH).} Performing numeric operations or solving algebraic equations/constraints.
  \item \textbf{Counting / Sets (COUNT).} Basic combinatorics, set relations/operations, and discrete enumerations.
  \item \textbf{Compare / Order (COMP).} Ranking or pairwise comparison tasks (greater/less, sorting by a criterion).
  \item \textbf{Geometry / Measurement (GEOM).} Reasoning about shapes, angles, areas/lengths, units/conversions, and geometric relations.
  \item \textbf{Topology / Matching / Pairing (MATCH).} Assignment, bijection/invariant-based pairing, or structure-preserving correspondence.
  \item \textbf{Textual Entailment / Logical Consistency (ENTAIL).} Checking whether statements are supported, contradicted, or mutually consistent with given text/evidence.
\end{itemize}

Each sample is labeled at all three levels; its leaf label is the concatenation \texttt{L1-L2-L3} (e.g., \texttt{RSYN-CLO-FLOW}). When ambiguity arises, we prioritize (i) the dominant \emph{task nature} (L1), then (ii) \emph{knowledge closure} (L2), and finally (iii) the primary \emph{reasoning primitive} (L3).

\subsection{More Proofs of Question Modification}\label{app:proofs}

\begin{lemma}
Let the original search graph be a single directed path
\[
  P = (v_{0}\to v_{1}\to\cdots\to v_{k}),
\]
which is the unique route from the start vertex $v_{0}$ to the goal
vertex $v_{k}$.
Embed an incompressible binary string $\tau$ of length
$|\tau| = \Delta I$ bits into the graph by attaching $m$ dead-end
(out-degree-one) edges while preserving $P$ as the \emph{only} goal
path.  Then
\[
  m \;\ge\; \Delta I - O(1).
\]
\end{lemma}

\begin{proof}
Fix a universal prefix Turing machine $U$.  Implicit in the lemma we
assume the embedding is performed by a fixed computable map
$\mathcal{E}:\{0,1\}^{\Delta I}\to\mathcal{G}$ that sends a bitstring
$\tau$ to a graph $G=\mathcal{E}(\tau)$ obtained from $P$ by attaching
$m$ dead-end edges while keeping $P$ as the unique goal path.
This ensures there is a fixed decoding procedure of constant size used
in the complexity argument below.

For each vertex $v_i$ on $P$ let $d_i$ be its out-degree in $G$ and set
$s_i := d_i-1\ge0$; thus
\[
  m \;=\; \sum_i s_i
\]
is the total number of added edges.  At vertex $v_i$ there are exactly
$d_i=1+s_i$ possibilities for which outgoing edge continues along $P$,
so the number of distinct graphs obtainable by choosing, at every
vertex, which outgoing edge is the path-edge is at most
\[
  \prod_i (1+s_i).
\]
Using the inequality $1+x\le 2^x$ (valid for all $x\ge0$) we get
\[
  \prod_i (1+s_i) \;\le\; \prod_i 2^{s_i} \;=\; 2^{\sum_i s_i} \;=\; 2^m.
\]
Hence there are at most $2^m$ distinct graphs that can result from
adding $m$ dead-end edges to $P$ while preserving $P$ as the unique
goal path.

Since $\mathcal{E}$ is a fixed computable embedding, different inputs
$\tau$ must produce different output graphs; therefore the number of
different $\tau$ representable with $m$ added edges is at most $2^m$.
It follows that $\tau$ has Kolmogorov complexity bounded by
\[
  K_U(\tau) \le m + O(1),
\]
where the $O(1)$ term accounts for the fixed-size description of the
decoding routine and the bookkeeping needed to recover $\tau$ from the
index of the graph.

On the other hand, by the incompressibility assumption
$K_U(\tau)\ge \Delta I - O(1)$.  Combining the two bounds yields
$m \ge \Delta I - O(1)$, as claimed.
\end{proof}

\subsection{More Details of Difficulty Adjustment}

\paragraph{Agent Recognition}

In the stage, \textsc{MorphoBench} adopts an image perturbation strategy based on the agent recognition of key visual information to increase task difficulty. We provide existing question–answer pairs to the agent and require it to identify and return the core visual elements within the corresponding images. Subsequently, as shown in Fig.~\ref{fig:example_reco}, we perform text processing on these key pieces of information by obfuscating their textual descriptions in the question and the image, thereby introducing interference at the textual level. 

In the process, we use the agent’s responses as the source of key visual information rather than relying on pre-defined annotations. The motivation is that allowing the model to indicate its most critical visual cues enables a more direct examination of its internal representations and attention mechanisms. In other words, when the visual elements recognized as critical are perturbed, its performance on the same task more faithfully reflects its robustness and reasoning capacity. Compared with externally imposed random noise, such agent-driven perturbations are more targeted and challenging, as they directly affect the features most relied upon. If a VLM continues to produce correct answers under such perturbations, it indicates robust fault tolerance and strong generalization. Conversely, a pronounced decline in performance reveals an excessive dependence on localized features and insufficient holistic understanding. Accordingly, this approach provides a more principled criterion for difficulty adjustment by assessing whether the model remains effective when the key features are perturbed.

\begin{figure}[h]
    \centering
    \begin{subfigure}{0.98\linewidth}
        \centering
        \includegraphics[width=\linewidth]{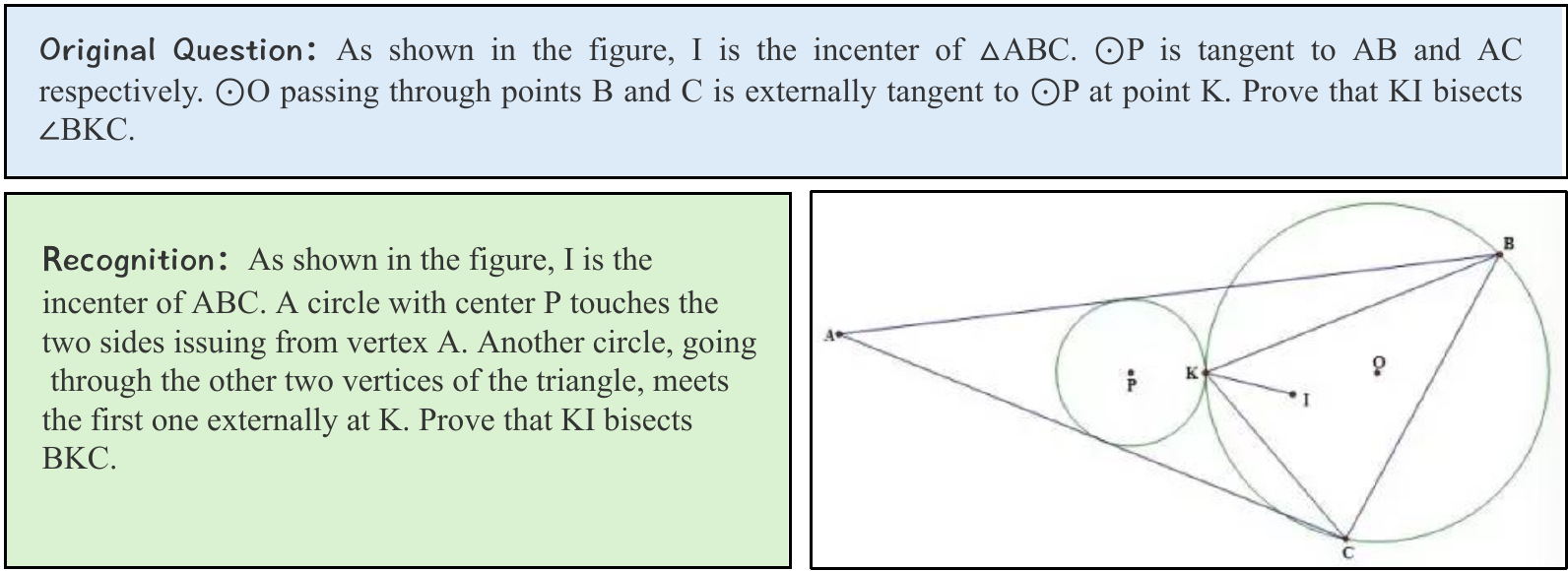}
        \caption{}
    \end{subfigure}

    \vspace{0.5em}

    \begin{subfigure}{0.98\linewidth}
        \centering
        \includegraphics[width=\linewidth]{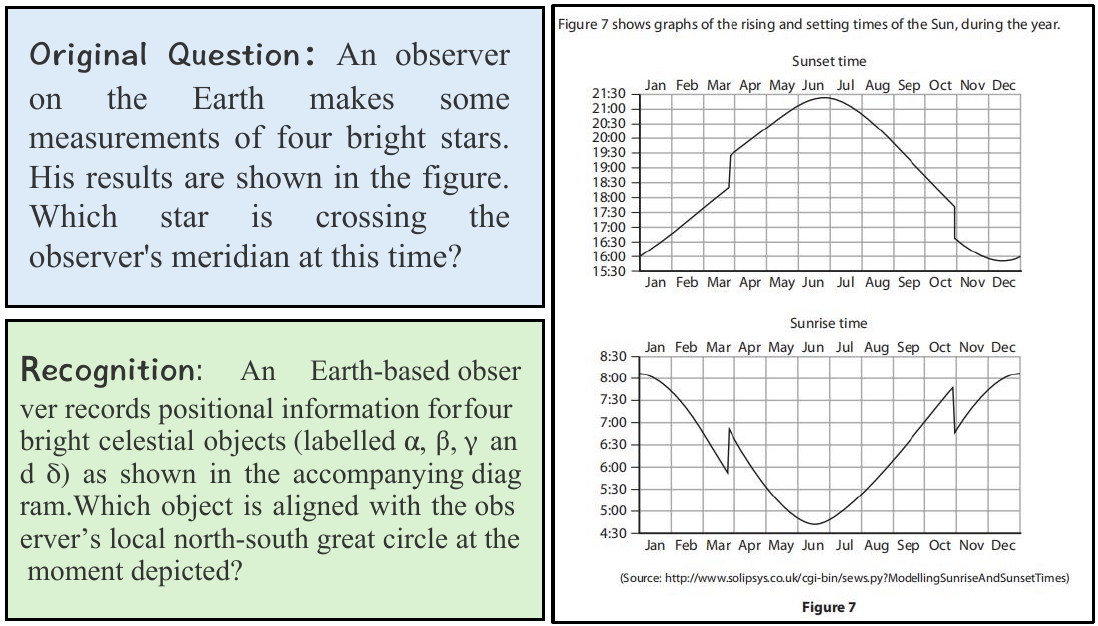}
        \caption{}
    \end{subfigure}

    \caption{Example for Agent Recognition.} 
    \label{fig:example_reco}
\end{figure}

\paragraph{Agent Reasoning}

Difficulty is a central factor in benchmark evaluation, yet it is often challenging to quantify due to its inherent subjectivity. Even when comparing problems within the same domain, it remains difficult to establish a rigorous partial order of difficulty; this challenge is further exacerbated when comparisons span across heterogeneous domains or disciplines. Conventional approaches typically resort to coarse-grained indicators—such as pass@N or weighted sums of chain-of-thought (CoT) lengths—which, while straightforward to compute, largely capture only superficial properties of model performance. Such measures fail to reflect more nuanced dimensions of reasoning, including the difficulty of exploration (the ability to branch into alternative solution paths), retrieval difficulty (the ability to identify relevant knowledge from prior context), and single-step reasoning difficulty (the precision of local logical inference).

To address these limitations, we propose the proof graph $G$. The graph serves as a modality that jointly encodes reasoning depth and exploration breadth, thereby offering a more fine-grained representation of problem-solving complexity. For a given model $M$, we refer to its underlying knowledge system as axioms, while the intermediate conclusions derived throughout the reasoning process are termed lemmas. Formally, the proof graph is defined as $G=(V,E)$, where $V$ denotes the set of lemmas and $E$ represents the directed edges capturing inferential dependencies between them. The reasoning sub-process can thus be viewed as the progressive activation of new lemmas, based on both the initial axioms and previously established lemmas.

\paragraph{Automatically Generated Questions}
\begin{figure*}[h]
    \centering
    \includegraphics[width=0.98\linewidth]{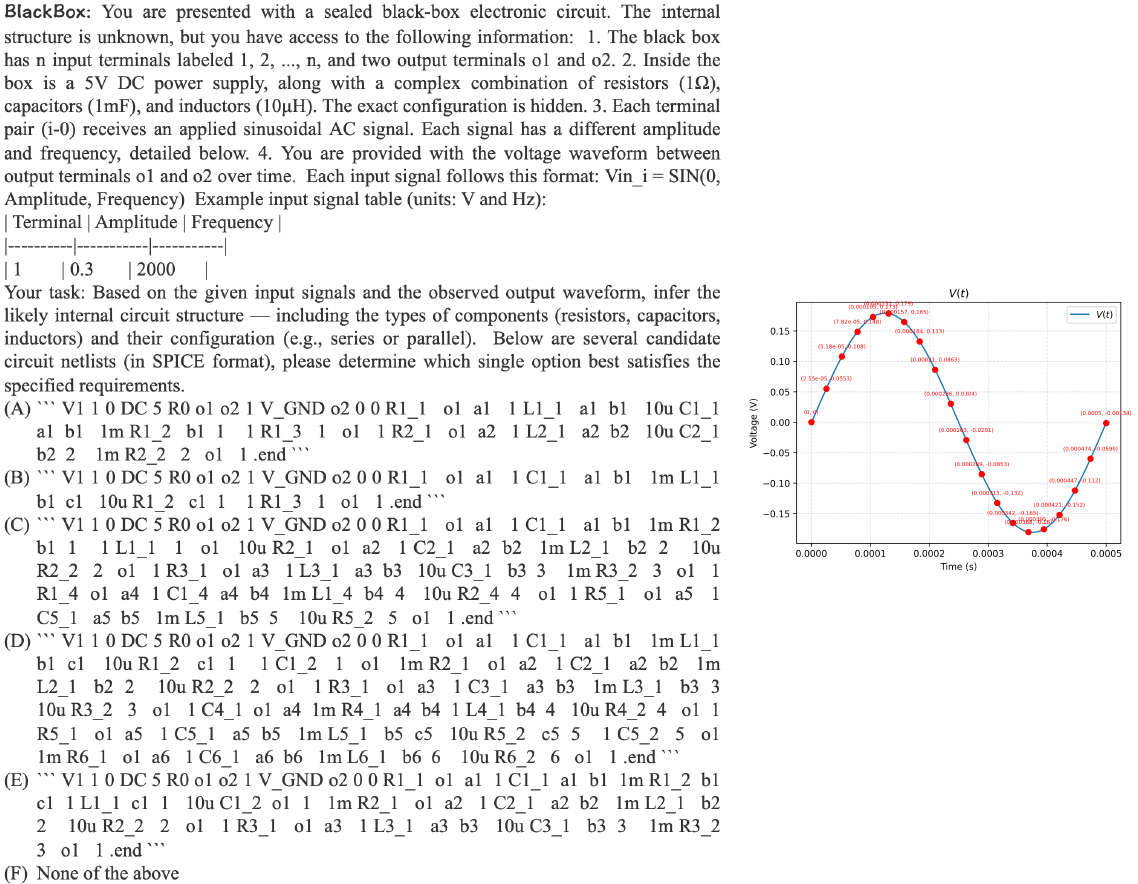}
    \caption{Example for Circuit Black-box Questions}
    \label{fig:example_blackbox}
\end{figure*}

In the \textsc{MorphoBench}, automatically generated questions constitute a crucial component of the benchmark. There are two main challenges: How to ensure the logicality, professionalism and the verification of the generated questions, and How to adjust the difficulty of the generated questions automatically.

To address the first challenge, we introduce external simulation software to ensure the correctness of the automatically generated questions. For the second challenge, we adjust key parameters of the automated question generation process to continuously increase both the complexity and the recognition difficulty of the tasks. Concretely, we design circuit black-box problems (in Fig.~\ref{fig:example_blackbox}) to evaluate reasoning ability and "spot the different one" tasks (in Fig.~\ref{fig:example_spot}) to assess visual recognition capacity . In circuit black-box problems, we leverage circuit simulators to validate outputs, producing waveform diagrams from output terminals to infer the underlying circuit structure. For difficulty adjustment, we control the number of external terminals exposed in the black-box. A larger number of terminals leads to higher difficulty. Although the internal structure is always theoretically solvable with the given component types, the complexity of the problem increases as the terminal count increases, making the reasoning task progressively more challenging. 
The “spot the different one” tasks present grids of visually similar characters (for example, Latin letters or Chinese characters), with exactly one character differing from the others, and the model is required to identify the outlier. The difficulty here is modulated either by selecting character pairs with greater visual similarity or by expanding the number of rows and columns. This setting probes the multimodal recognition capacity of VLM in a controlled manner.

These mechanisms not only ensure the quality of automatically generated questions, but also support the design goal of \textsc{MorphoBench}. The benchmark aims to realize self-evolving difficulty: by expanding terminal counts in circuits or grid size and similarity in visual tasks, the dataset naturally evolves toward harder problems. This allows the benchmark to continually stretch the boundaries of existing models, probing the upper limits of reasoning and multimodal understanding. By embedding evolutionary adjustment of difficulty into the generation pipeline, \textsc{MorphoBench} establishes a dynamic and extensible evaluation platform, maintaining long-term relevance as models advance.

\begin{figure*}[h]
    \centering
    \includegraphics[width=0.98\linewidth]{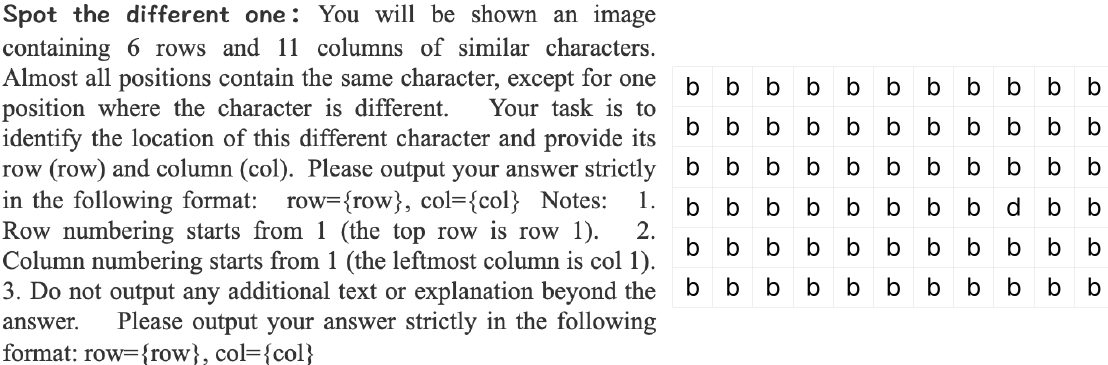}
    \caption{Example for "Spot the Different One"}
    \label{fig:example_spot}
\end{figure*}

\section{More Information}

We collected data from two main sources: the Art of Problem Solving (AoPS) website, Chinese Mathematics Olympiad (CMO) Training Problems and Chinese Physics Olympiad (CPhO) Training Problems. Both sources already provide complete solutions or official answers, so no additional human annotation was required. To ensure data quality, we conducted manual verification of the collected materials. The human checkers responsible for this process were compensated at approximately USD 570 per month. We also obtained permission from the respective data providers before using their materials for research purposes. These resources were chosen because they are authoritative, widely used in mathematics training, and highly relevant to the high-school and olympiad-level problem domain addressed in our study.

We relied on widely used benchmark datasets that have long served as standard resources in the research community. These datasets are curated by reputable organizations, and to the best of our knowledge, they do not include personal identifiers or inappropriate material. They are distributed under established usage policies, and any elements with potential sensitivity have already been anonymized or excluded. For these reasons, we concluded that no further anonymization or additional data checks were necessary for our work.

An AI assistant was employed solely for grammar correction and minor stylistic improvements. It was not involved in the design, analysis, or technical development of the research. Consequently, no additional disclosure regarding AI assistance was required in the paper.

\begin{table*}[htbp]
\centering
\small
\setlength{\tabcolsep}{5pt}
\begin{tabular}{lccccc|c}
\toprule
\textbf{Model} & \textbf{Mathematics} & \textbf{Engineering} & \textbf{Natural Sci.} & \textbf{Social Sci.} & \textbf{Other} & \textbf{\textsc{Morpho}-v0} \\
\midrule
Total (share) & 552 (42.23\%) & 220 (16.83\%) & 250 (19.13\%) & 91 (6.96\%) & 194 (14.85\%) & -- \\
\midrule
claude-4           & 34.11 & \underline{37.58} & 17.20 & 46.51 & 6.13 & 29.22 \\
gemini-2.5-flash   & 41.85 & 17.27 & 28.00 & \underline{61.54} & 36.60 & 35.65 \\
gemini-2.5-pro     & 43.30 & 7.73 & 28.00 & \textbf{67.03} & 34.02 & 34.66 \\
gpt-5     & \textbf{57.53} & 36.82 & \underline{29.20} & 52.75 & \underline{37.63} & \underline{45.33} \\
grok-4             & 49.11 & 5.47 & 16.00 & 52.33 & 1.89 & 29.55 \\
o3 & \underline{53.26} & \textbf{37.73} & \textbf{34.40} & 56.04 & \textbf{41.75} & \textbf{45.52} \\
o4-mini            & 51.81 & 13.64 & 27.60 & 48.35 & 32.99 & 37.72 \\
\bottomrule
\end{tabular}
\caption{
Cross-disciplinary performance on the \textsc{Morpho}-v0.
Each column reports the accuracy (\%) of reasoning models across aggregated subject categories.  
The final column denotes the weighted overall accuracy based on the sample proportion of each subject group.  
\textbf{Boldface} marks the best value per column, and \underline{underline} indicates the second-best.
}
\label{tab:per_model_5cat}
\end{table*}

\section{More Evaluation Results}
\subsection{Cross-disciplinary Analysis}\label{app:Cross-disciplinary}

As shown in Table~\ref{tab:per_model_5cat}, the cross-disciplinary performance of different models demonstrates distinct domain preferences and weaknesses. Notably, while O3 and GPT‑5 achieve the most balanced and overall highest accuracies across the five subject groups, other models exhibit pronounced inconsistency between formal and applied domains.

For instance, Grok‑4 attains a high score in Mathematics (49.11\%), indicating strong capability in symbolic manipulation and formal reasoning. However, its accuracy in Engineering drops sharply to only 5.47\%, suggesting poor generalization to problem-solving contexts that involve applied physical reasoning or multiple-step procedural understanding. This drastic imbalance significantly drags down its overall weighted accuracy compared with top-performing models.

Conversely, Gemini‑2.5‑Pro and Claude‑4 display moderate performance concentrated in social and conceptual domains, yet their Engineering accuracy (7.73\% and 37.58\%, respectively) reveals clear limitations in applied reasoning. GPT‑5 maintains high accuracy in both Mathematics (57.53\%) and Social Sciences (52.75\%), demonstrating adaptability to both formal derivation and contextual inference tasks. Overall, the observed domain-specific disparities emphasize that frontier reasoning models, despite improving generalization in linguistic and conceptual domains, still face major challenges in transferring symbolic reasoning capabilities to applied and domain-specific problem settings.

\subsection{Visualized Examples of Agent Recognition and Agent Reasoning Adjustments}

To further demonstrate the adaptability and generality of our benchmark, we present representative examples under the two proposed difficulty adjustment paradigms: agent recognition and agent reasoning.

In the agent recognition adjustment, difficulty is modulated through textual fuzzification guided by visual grounding. Specifically, the model first identifies the key visual elements that support the correct answer, such as symbols, numbers, geometric labels, or local regions, and then weakens or replaces the corresponding textual expressions in the question with qualitative descriptions. This process preserves solvability while increasing ambiguity, compelling models to rely more on visual perception rather than direct text–answer mapping.

In contrast, the agent reasoning adjustment focuses on the cognitive chain of inference. By analyzing the essential theorems and intermediate steps within the reasoning process, we strategically introduce irrelevant or partially related hints to interfere with the model’s logical flow. These additions encourage the model to distinguish between critical and misleading information, thereby evaluating its structured reasoning ability under uncertainty.

In the following examples~\cref{fig:agent_examples1,fig:agent_examples2,fig:agent_examples3,fig:agent_examples4}, we visualize several representative instances to illustrate these two adjustment modes. These multi-disciplinary examples collectively demonstrate how our benchmark dynamically reconfigures question difficulty through two complementary mechanisms, enabling more fine-grained and interpretable evaluation of multimodal reasoning capabilities.
\begin{figure*}[t]
  \centering
  \includegraphics[width=0.98\linewidth]{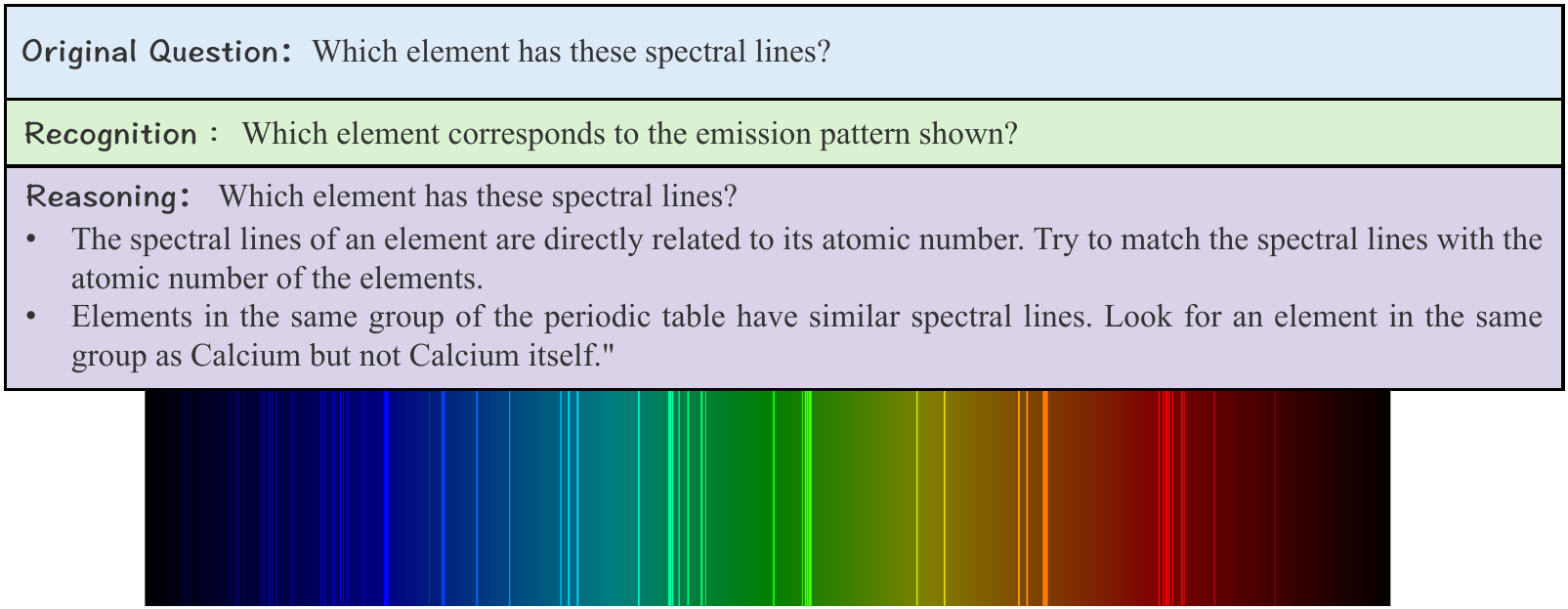}
  \caption{Multi-disciplinary examples under agent recognition and reasoning.}
  \label{fig:agent_examples1}
\end{figure*}

\begin{figure*}[t]
  \centering
  \includegraphics[width=0.9\linewidth]{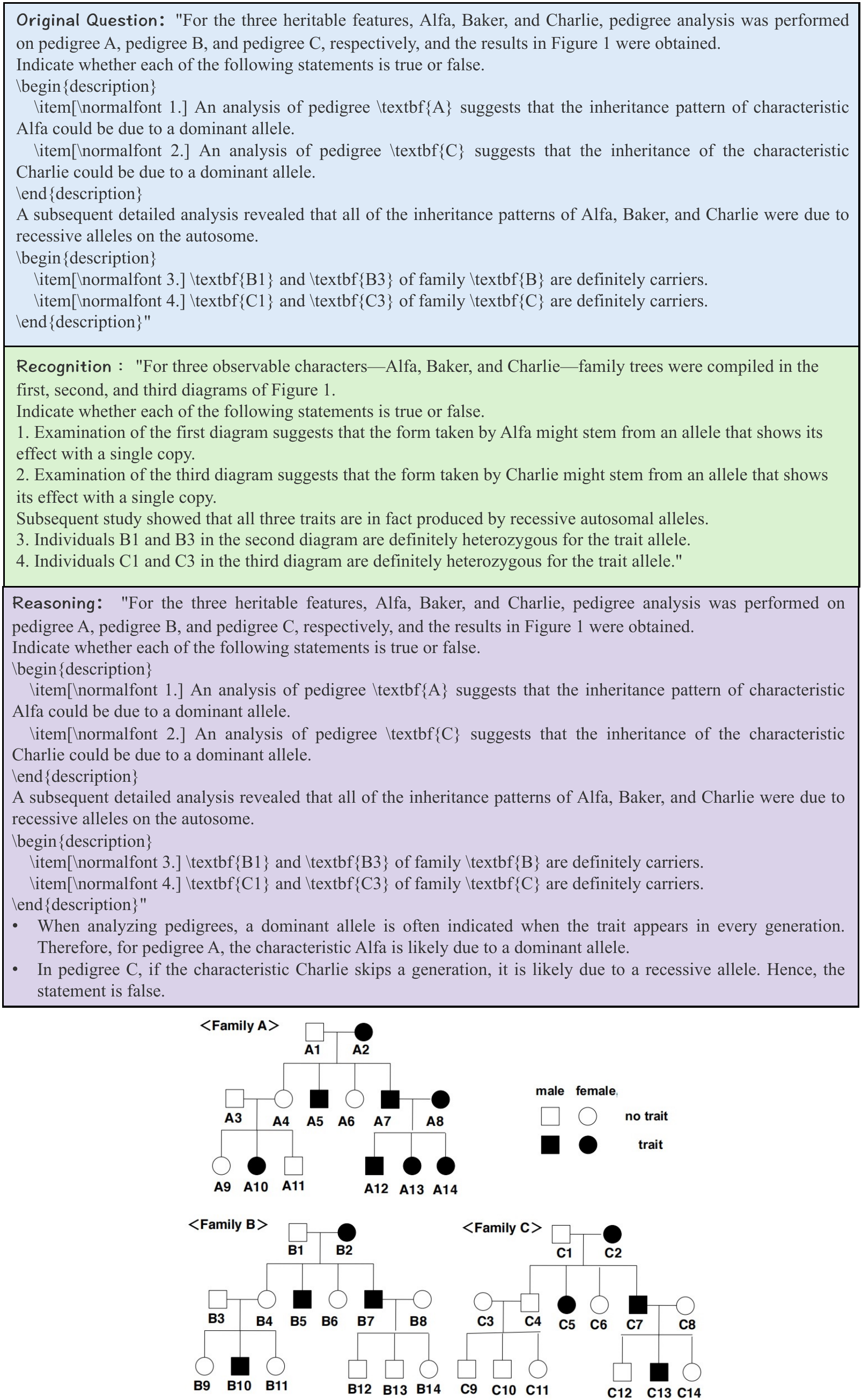}
  \caption{Multi-disciplinary examples under agent recognition and reasoning.}
  \label{fig:agent_examples2}
\end{figure*}

\begin{figure*}[t]
  \centering

  \begin{subfigure}{0.98\linewidth}
    \centering
    \includegraphics[width=\linewidth]{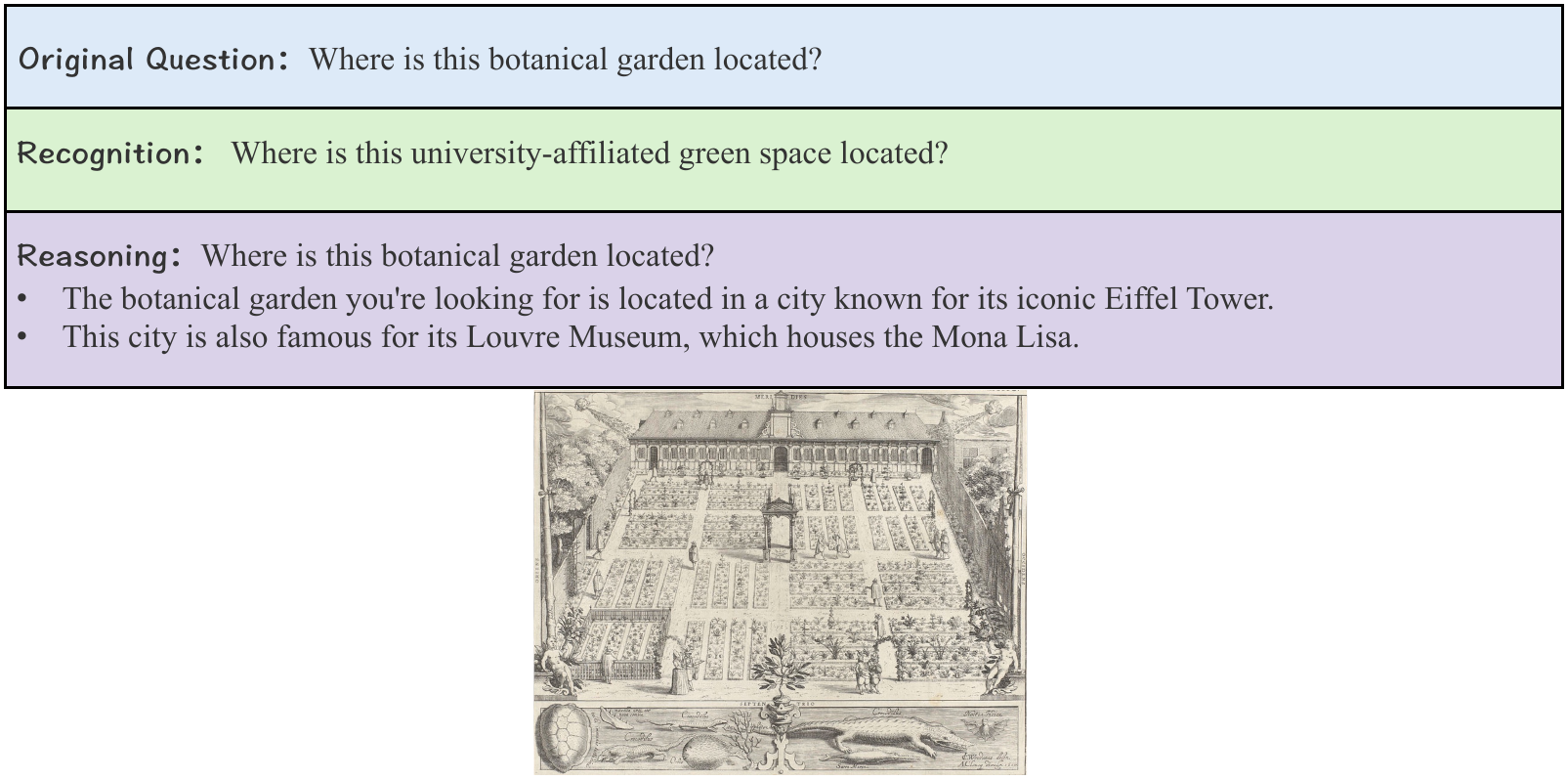}
    \caption{}
  \end{subfigure}

  \vspace{0.6em}

  \begin{subfigure}{0.98\linewidth}
    \centering
    \includegraphics[width=\linewidth]{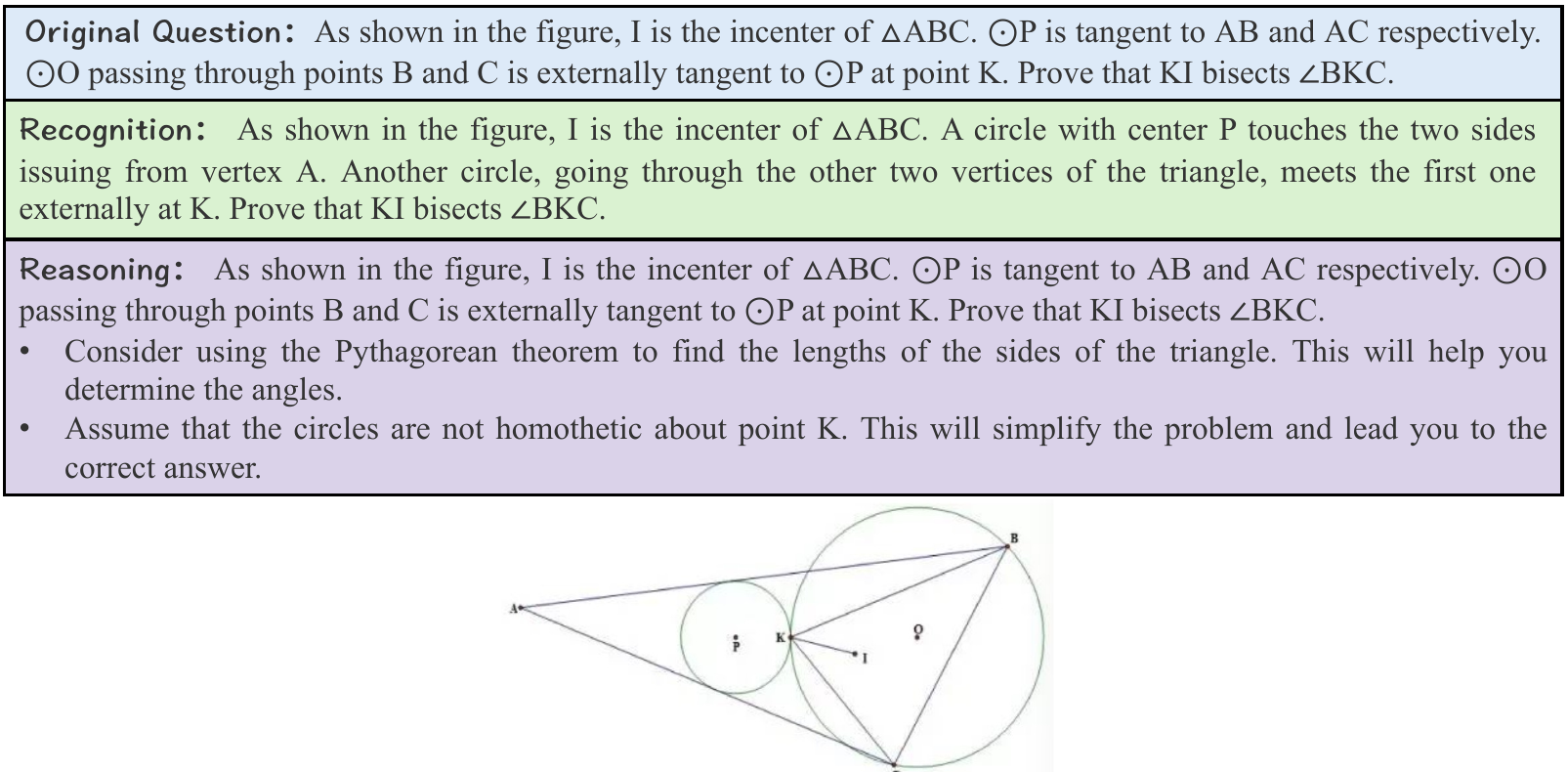}
    \caption{}
  \end{subfigure}

  \caption{Multi-disciplinary examples under agent recognition and reasoning.}
  \label{fig:agent_examples3}
\end{figure*}

\begin{figure*}[t]
  \centering

  \begin{subfigure}{0.98\linewidth}
    \centering
    \includegraphics[width=\linewidth]{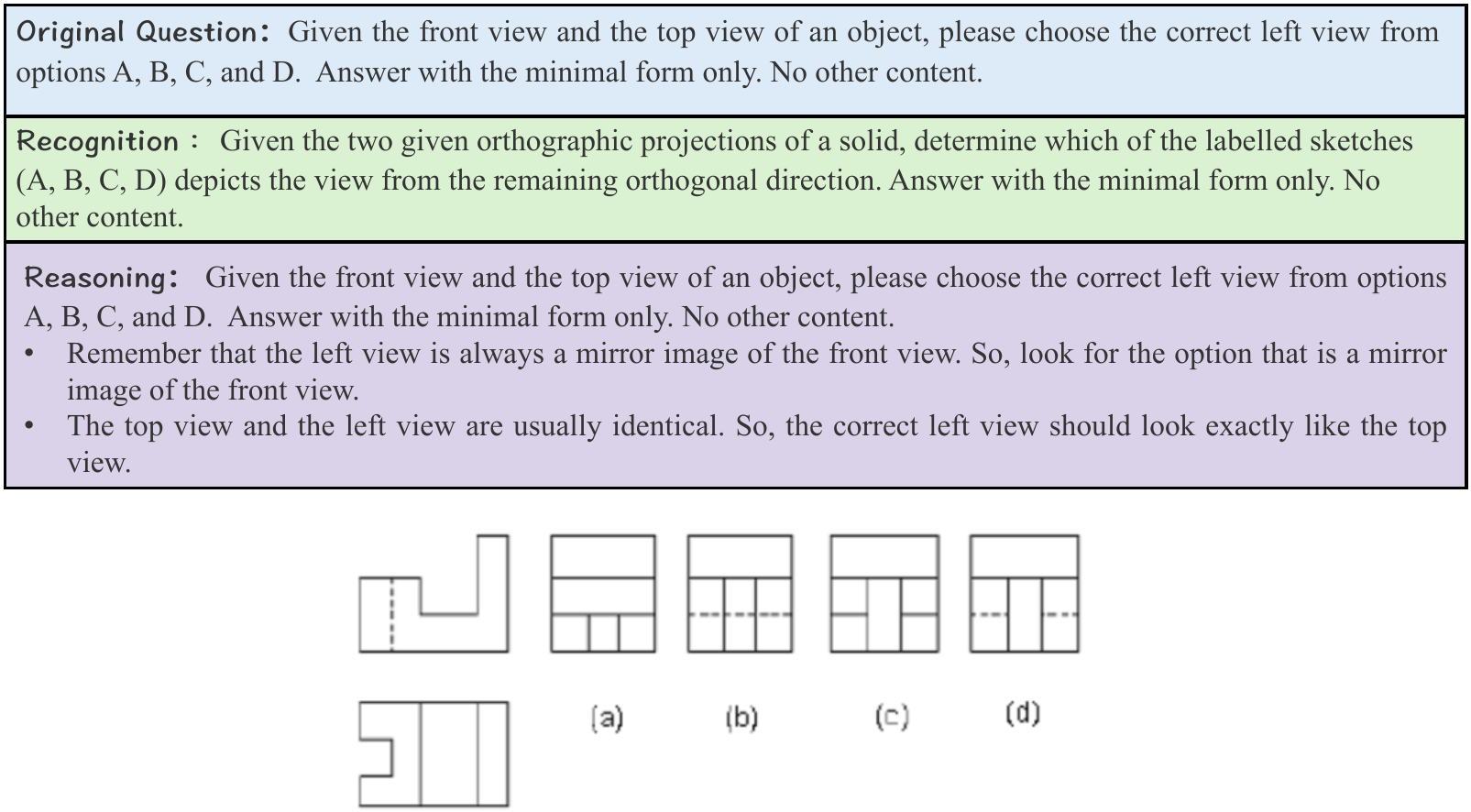}
    \caption{}
  \end{subfigure}

  \vspace{0.6em}

  \begin{subfigure}{0.98\linewidth}
    \centering
    \includegraphics[width=\linewidth]{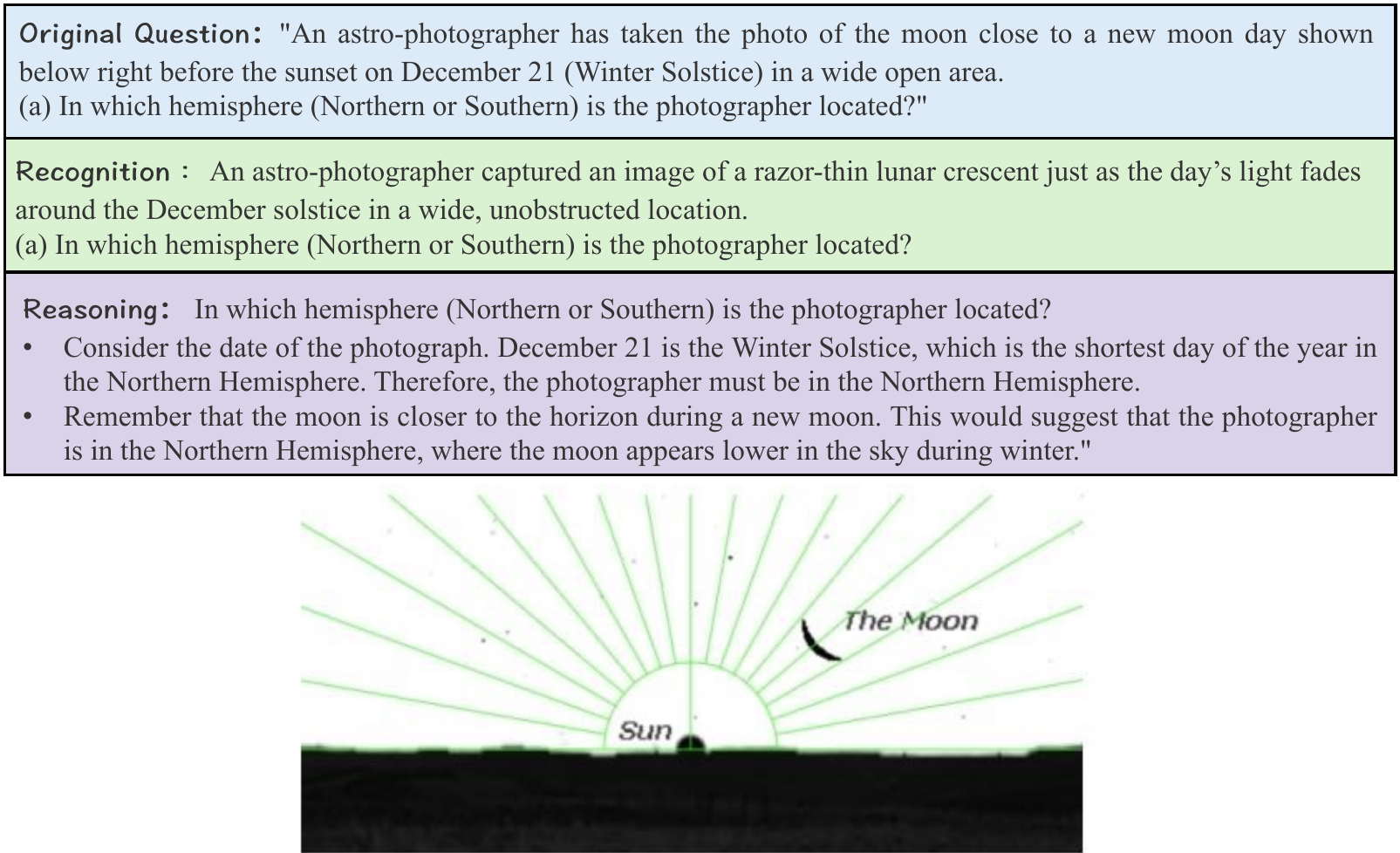}
    \caption{}
  \end{subfigure}
  \caption{Multi-disciplinary examples under agent recognition and reasoning.}
  \label{fig:agent_examples4}
\end{figure*}

\section{Broader Impact}

\subsection{Societal Impact}
We propose \textsc{MorphoBench}, a high-quality multidisciplinary reasoning benchmark that provides a robust standard for evaluating the reasoning capabilities of state-of-the-art models. It has no direct negative societal impacts. However, we must also be cautious about the potential misuse of \textsc{MorphoBench} by unlawful individuals.

\subsection{Future Work}

In this work, we propose \textsc{MorphoBench}, a multidisciplinary large model reasoning benchmark. This benchmark not only encompasses a wide variety of problem types but also allows for dynamic difficulty adjustment based on the model's reasoning capabilities. However, although modifying test questions according to the model's reasoning process appropriately tailors the difficulty to the model's abilities, it still falls short of generating entirely novel scientific reasoning problems. In the future, we will continue to build upon our current research direction by leveraging the limitations observed in model reasoning to enable automated generation of new questions based on reference literature.

\subsection{Potential Risks}
\label{sec:risks}

\begin{itemize}
  \item \textbf{Sources and compliance.} All items come from publicly available datasets, competitions, or exams with proper citations.
  \item \textbf{Privacy and safety.} The dataset contains no personally identifiable information; any flagged sensitive content will be anonymized or removed.
  \item \textbf{Exam/contest leakage risk.} Some problems originate from past or mock exams and could be misused for “drill” purposes.
  \item \textbf{AI-assisted writing disclosure.} LLMs were used only for language polishing and formatting of the paper, not for answer annotation or drawing conclusions; all labels were verified by human experts.
\end{itemize}

\noindent\textit{Note.} Data collection and annotation procedures are summarized in the main text of paper; all sources are public and cited in the main text or Appendix.

\end{document}